\documentclass[runningheads,envcountsame,envcountsect]{llncs}

\usepackage[utf8]{inputenc}
\usepackage{geometry}
\usepackage{graphicx}
\usepackage{setspace}
\usepackage{amsmath}
\usepackage{proof}
\usepackage{url}
\usepackage{hyperref}
\usepackage{amssymb}
\usepackage{mathtools}
\usepackage{enumitem}
\usepackage{centernot}
\usepackage{tikz}
\usetikzlibrary{shapes.geometric}
\usetikzlibrary{patterns}
\usepackage{appendix}               
\usepackage{thmtools}

\newcommand{\co}[1]{\overline{#1}}
\newcommand{\parop}{\,|\,}
\newcommand{\res}[1]{(\nu #1)}
\newcommand{\fn}[1]{\mathtt{fn}(#1)}

\newcommand{\nil}{0}
\newcommand{\rel}{\mathrel{\mathcal{R}}}

\newcommand{\arrof}[1]{\xrightarrow{#1}_f}
\newcommand{\arrob}[1]{\xrightarrow{#1}_r}
\newcommand{\arro}[2]{\xrightarrow{#1} _{#2}}

\newcommand{\Arro}[2]{\overset{#1}{\Rightarrow} _{#2}}
\newcommand{\arroccs}[1]{\xrightarrow{#1}_{ccs}}

\newcommand{\sub}[2]{[#1/#2]}

\newcommand{\keys}[1]{\mathtt{keys}(#1)}
\newcommand{\std}[1]{\mathtt{std}(#1)}

\newcommand{\tostd}[1]{\mathtt{toStd}(#1)}

\newcommand{\delhist}[1]{\mathtt{delHist}(#1)}

\newcommand{\wsim}{\approx}
\newcommand{\swsim}{\cong}
\newcommand{\frsim}{\sim_{FR}}

\newcommand{\Keys}{\mathtt{Keys}}
\newcommand{\Rule}[2]{\displaystyle{\frac{#1}{#2}}}
\newcommand{\key}[1]{[#1]}

\newcommand{\tr}[2]{\iftr #1 \else #2 \fi}
\newif\iftr

\pgfdeclarepatternformonly{north east spaced lines}{\pgfqpoint{0pt}{0pt}}{\pgfqpoint{80pt}{80pt}}{\pgfqpoint{80pt}{80pt}}%
{
    \pgfsetlinewidth{0.65pt}
    \pgfpathmoveto{\pgfqpoint{0pt}{0pt}}
    \pgfpathlineto{\pgfqpoint{103pt}{103pt}}
    \pgfusepath{stroke}
}

\trfalse 

\tr{
  \usepackage[appendix=append]{apxproof}
  }
   {
     \usepackage[appendix=strip]{apxproof}
     }

\begin{document}

\title{On Weak Bisimilarities in CCSK}
\titlerunning{On Weak Bisimilarities in CCSK} 
%
\author{Baptiste Vallée\inst{1}\orcidID{0009-0002-1268-3761} \and Ivan Lanese\inst{2}\orcidID{0000-0003-2527-9995}\\
\email{baptistevalleedupont@gmail.com, ivan.lanese@gmail.com}}
\authorrunning{B.~Vallée, I.~Lanese} 
\institute{École Normale Supérieure Paris-Saclay (France)\and Olas Team, University of Bologna/INRIA (Italy)}

\maketitle   
This is a pre-copy-editing, author-produced PDF of an article accepted for publication in RC 2026 following peer review. The definitive publisher-authenticated version is available online at https://link.springer.com/chapter/10.1007/978-3-032-30839-9
\begin{abstract}
In the context of CCSK, a reversible extension of CCS, we study
different notions of bisimilarity (strong/weak,
forward-only/reversible) and highlight their differences and
commonalities. In particular, for the weak reversible case, not previously
studied in the literature, we propose two variants, dubbed directional and mixed bisimilarity, depending on
whether $\tau$ actions should be in the same direction
(forward/backward) as the action being matched or not.
We show, in particular, that mixed bisimilarity is a congruence and completely abstracts away from $\tau$ actions.
\keywords{CCS, CCSK, Reversible Computation, Weak Bisimilarity, Behavioral Theory}
\end{abstract}

\section{Introduction}
Building concurrent systems is challenging due to the complexity of
reasoning about numerous possible interleavings, yet concurrency is
essential in modern systems like the Internet, cloud computing, and
parallel processing. Reversible computing, which allows systems to
execute both forwards and backwards, recovering past states, has
significant applications in low-energy
computing~\cite{landauer1961irreversibility},
simulation~\cite{carothers1999efficient}, biological
modeling~\cite{cardelli2011reversibility,phillips2012reversible}, and
program
debugging~\cite{engblom2012review,mcnellis2017time,lanese2018cauder}. Many
of these applications involve concurrent systems, leading to the
development of reversible extensions of concurrent process calculi
such as CCS~\cite{danos2004reversing,phillips2007reversing} and
the \(\pi\)-calculus~\cite{CristescuKV13}, and even of concurrent
programming languages such as Erlang~\cite{LaneseNPV18} and Go~\cite{OguchiYY25}.

A main notion in the theory of process calculi is the notion of
bisimilarity~\cite{Sangiorgi2011}, allowing one to prove two processes
equivalent, e.g., to prove an implementation equivalent to a more
abstract specification. In particular, bisimilarity requires
equivalent processes to be able to match each other actions, and in
doing so going to processes which are still equivalent.  While strong
and weak bisimilarity (weak bisimilarity differs from the strong one
as the former abstracts away from internal actions, focusing only on
interactions with the context) have been extensively studied in
concurrent systems, the literature lacks an analysis of weak
bisimilarities in a reversible setting. Our study addresses this gap by
investigating the relationships between different notions of
bisimilarity (strong/weak, forward-only/reversible) in the context of
CCSK~\cite{phillips2007reversing}, a causal-consistent reversible
extension of Milner CCS~\cite{milner1989communication}.  In
particular, in the definition of weak reversible bisimilarity, a main
decision is whether auxiliary $\tau$ actions (representing internal steps) allowed in the simulation
of some action $\alpha$ need to be in the same direction as $\alpha$
or not. The two alternatives lead to different equivalences.


We consider this work as a first step in the exploration of weak
bisimilarity in a reversible setting, paving the way for a deeper
exploration in the future. We claim as our main contributions the
proposal of two notions of weak reversible bisimilarity (mixed
bisimilarity in Definition~\ref{def:mbis} and directional bisimilarity
in Definition~\ref{def:dbis}, both in Section~\ref{sec:bisim}), the study of
the relations between different notions of bisimilarity in CCSK
(Section~\ref{sec:relations}), and the study of which of these notions
are congruences
(Section~\ref{sec:properties}). In particular, we show that mixed
bisimilarity is a congruence (Theorem~\ref{th:mixediscongr}) and
completely abstracts away from $\tau$ actions
(Proposition~\ref{prop:mixedtau} and Theorem~\ref{th:axiom_approx_m}).
Another surprising result is that extending CCS bisimilarities to CCSK
gives equivalences which are not congruences
(Proposition~\ref{prop:fnocongr}), even when they are congruences in
CCS, as in the case of strong bisimilarity.

\section{CCSK}
In this section we recall the main elements of CCSK, while referring to~\cite{phillips2007reversing} for further details.
We assume an infinite set of \textit{Names} $\mathcal{A}$, ranged over by $a,b,c,\ldots$, and a disjoint infinite set of \textit{Co-names} $\overline{\mathcal{A}}$, ranged over by $\overline{a},\overline{b},\overline{c},\ldots$, where $ \overline{*} $ is an operator such that $\overline{\overline{a}} = a$.
We call \textit{actions} the elements of $ \mathcal{A} \cup \overline{\mathcal{A}} \cup \{\tau\}$ where $\tau \notin \mathcal{A}$ and $\overline{\tau}$ is undefined,
ranged over by $ \alpha, \beta, \ldots$ 
Intuitively, names represent input actions, co-names represent output actions, and $\tau$ is an internal synchronisation.
CCS processes, which we shall also call \emph{standard} processes, are given by:
\[ P,Q := 0 \mid\mid \alpha.P \mid\mid P+Q \mid\mid P \parop Q \mid\mid \res a P \]
Intuitively, $0$ is the inactive process, $\alpha.P$ is a process that
performs action $\alpha$ and continues as $P$, $P+Q$ is
nondeterministic choice, $P \parop Q$ is parallel composition and
restriction $\res a P$ binds name $a$ and the corresponding co-name $\overline{a}$ inside $P$. A name
is \emph{bound} if it is inside the scope of a restriction
operator, \emph{free} otherwise.
Function $\fn{P}$
computes the set of free names in process $P$.
We set this convention: unary operators bind stronger than binary operators.

CCSK extends CCS with the possibility of executing backwards. In order to
remember which input interacted with which output while going forwards,
fresh \emph{keys} are created at each forward step, and the same key is
used to label an input and the corresponding output during a synchronisation.

We denote the set of keys by $\Keys$, ranged over by $m,n,k,\ldots$.
Prefixes, ranged over by $\pi$, are of the form $\alpha[m]$ or
$\alpha$. The former denotes that $\alpha$ has already been executed,
the latter that it has not. 

CCSK processes are given by:
\[ P,Q := \nil \mid\mid \pi.P \mid\mid P+Q \mid\mid P \parop Q \mid\mid \res a P \]
hence they are like CCS processes but for the fact that prefixes
may be labelled with a key. In the following, we may drop trailing $0$s.

\begin{definition}[Context]
A CCSK \emph{context} is a process with a hole, as generated by the
grammar below:
\[
C := \bullet \mid\mid \pi.C \mid\mid C+Q \mid\mid P+C \mid\mid C \parop Q \mid\mid P \parop C \mid\mid \res a C
\]
We denote with $C[P]$ the process obtained by replacing $\bullet$ with
$P$ inside $C$.
\end{definition}

We use predicate $\std{P}$ to mean that $P$ is standard, that
is none of its actions has been executed, hence it has no keys. We
assume function $\tostd{P}$ that takes a CCSK process $P$ and gives
back the standard process obtained by removing all keys from $P$.

We take from~\cite[Def.~2.1]{lanese2021forward} the notions of free
and bound keys.

\begin{definition}[Free and bound keys]\label{def:free}
A key $k$ is \emph{bound} in a process $X$ iff it occurs either twice,
attached to complementary prefixes, or once, attached to a
$\tau$ prefix. A key $k$ is \emph{free} if it occurs once, attached to
a non-$\tau$ prefix.
\end{definition}

Figure~\ref{ForwardCCSK} shows the forward rules of CCSK.  Backward
rules in Figure~\ref{ReverseCCSK} are obtained from forward rules by
reversing the direction of transitions.  Both relations rely on a
definition of structural congruence allowing one to $\alpha$-convert
bound keys, applicable only at top level (this condition is needed to ensure that there are no other occurences of $n$ in the context):

\[P \equiv P\sub{n}{m} \qquad m \text{ bound in } P, n \notin \keys{P}\]

\begin{figure}[t] \[
\begin{array}{c}
\text{(TOP)}\quad \Rule {\std{P}} {\alpha.P \arrof{\alpha\key m}
  \alpha\key m.P} \qquad \text{(PREFIX)}\quad \Rule {P
  \arrof{\beta\key n} P'} {\alpha\key m.P \arrof{\beta\key n}
  \alpha\key m.P'} \; m \neq n
\\[15pt]
\text{(CHOICE)}\quad
\Rule
{P \arrof{\alpha \key m} P' \quad \std{Q}}
{P+Q \arrof{\alpha \key m} P'+Q}
\qquad
\Rule
{Q \arrof{\alpha \key m} Q' \quad \std{P}}
{P+Q \arrof{\alpha \key m} P+Q'}
\\[15pt]
\text{(PAR)}\quad
\Rule
{P \arrof{\alpha\key m} P'\quad m \notin \keys Q }
{P \parop Q \arrof{\alpha\key m} P' \parop Q}
\qquad
\Rule
{Q \arrof{\alpha\key m} Q' \quad m \notin \keys P }
{P \parop Q \arrof{\alpha\key m} P \parop Q'}\\[15pt]
\text{(SYNCH)}\quad
\Rule
{P \arrof{\alpha\key m} P' \quad Q \arrof{\co {\alpha} \key m} Q'}
{P \parop Q \arrof {\tau\key m}  P' \parop Q'}
\quad (\alpha \neq \tau)
\\[15pt]
\text{(RES)}\quad
\Rule
{P \arrof {\alpha\key m} P'}
{\res a P  \arrof {\alpha\key m} \res a P'}
\ \alpha \notin \{a,\co a\}
\\[15pt]
\text{(EQUIV)}\quad
\Rule
{P \equiv Q \quad Q \arrof{\alpha \key m} Q' \quad Q' \equiv P'}
{P \arrof{\alpha \key m} P'}\;
\begin{array}{c}
\text{applicable only}\\ \text{at top level}
\end{array}
\end{array}
\] \caption{Forward SOS rules for CCSK} \label{ForwardCCSK}
\end{figure}
\begin{figure}[t] \[
\begin{array}{c}
\text{(BK-TOP)}\quad
\Rule
{\std{P}}
{\alpha\key m.P \arrob{\alpha\key m} \alpha.P}
\qquad
\text{(BK-PREFIX)}\quad
\Rule
{P \arrob{\beta\key n} P'}
{\alpha\key m.P \arrob{\beta\key n} \alpha\key m.P'}
\; m \neq n
\\[15pt]
\text{(BK-CHOICE)}\quad
\Rule
{P \arrob{\alpha \key m} P' \quad \std{Q}}
{P+Q \arrob{\alpha \key m} P'+Q}
\qquad
\Rule
{Q \arrob{\alpha \key m} Q' \quad \std{P}}
{P+Q \arrob{\alpha \key m} P+Q'}
\\[15pt]
\text{(BK-PAR)}\quad
\Rule
{P \arrob{\alpha\key m} P'\quad m \notin \keys Q }
{P \parop Q \arrob{\alpha\key m} P' \parop Q}
\qquad
\Rule
{Q \arrob{\alpha\key m} Q' \quad m \notin \keys P }
{P \parop Q \arrob{\alpha\key m} P \parop Q'}\\[15pt]
\text{(BK-SYNCH)}\quad
\Rule
{P \arrob{\alpha\key m} P' \quad Q \arrob{\co {\alpha} \key m} Q'}
{P \parop Q \arrob {\tau\key m}  P' \parop Q'}
\quad (\alpha \neq \tau)
\\[15pt]
\text{(BK-RES)}\quad
\Rule
{P \arrob {\alpha\key m} P'}
{\res a P  \arrob {\alpha\key m} \res a P'}
\ \alpha \notin \{a,\co a\}
\\[15pt]
\text{(BK-EQUIV)}\quad
\Rule
{P \equiv Q \quad Q \arrob{\alpha \key m} Q' \quad Q' \equiv P'}
{P \arrob{\alpha \key m} P'}\;
\begin{array}{c}
\text{applicable only}\\ \text{at top level}
\end{array}
\end{array}
\] \caption{Reverse SOS rules for CCSK} \label{ReverseCCSK}
\end{figure}

Rule (TOP) allows a prefix to execute. The rule generates a key
$m$. Freshness of $m$ is guaranteed by the side conditions of the
other rules (cf.~rule (PAR)). Rule (PREFIX) states that an executed prefix does not
block execution. The two rules for (CHOICE) and the two for (PAR)
allow processes to execute inside a choice or a parallel
composition. The side condition of rule (CHOICE) ensures that at most
one branch can execute. Rule (SYNCH) allows two complementary actions
to synchronise producing a $\tau$. The key of the two actions needs to
be the same. Rule (RES) allows an action which does not involve the
restricted name to propagate through restriction.

The forward semantics of a CCSK process is the smallest relation
$\arrof{}$ closed under the rules in Figure~\ref{ForwardCCSK}.
Analogously, its backward semantics
is the smallest relation
$\arrob{}$ closed under the rules in Figure~\ref{ReverseCCSK}.
The semantics is the union of the two relations.
From now on, we let $\vartheta$ range over $\alpha[m]$ and $\mu$ range over $\alpha[m]$ with $\alpha \neq \tau$. Let $x,y,\ldots$ range over the set of directions $\{f,r\}$, for forward and reverse.

As standard in reversible computing (see, e.g., \cite{phillips2007reversing} or
the notion of coherent process in~\cite{danos2004reversing}), all the
developments consider only processes reachable from a standard
process.

\begin{definition}[Reachable process]\label{def:reachable}
  A process $Q$ is \emph{reachable} iff there exists a standard process $P$
  and a finite sequence of transitions from $P$ to $Q$.
\end{definition}

\section{Bisimilarities}\label{sec:bisim}
In this section we introduce the various notions of bisimilarity we
study. We start from CCS bisimilarities, and then move to
bisimilarities specific for CCSK. Note that CCS can be seen as a subset of CCSK,
considering only standard processes and only forward semantics.
Hence, we will extend CCS bisimulations to CCSK by just considering
forward transitions of CCSK terms. As a consequence, history becomes
inaccessible, hence ideally irrelevant. However, requiring that a
transition is matched by a transition with identical label would leak
information on which keys are used in the history, as shown by the
following example.

\begin{example}
  We have $b \arrof{b[n]} b[n]$ while no transition with the same
  label is enabled from $a[n].b$.  The second process hence cannot
  match the transition of the first one in the bisimulation game if
  equality of labels is required. Hence, keys of forward transitions
  leak information about which keys are used in the history.
\end{example}

In order to avoid this issue, we will remove keys from the labels.
Hence, we define CCS semantics for CCSK processes as follows:
\begin{definition}[CCS semantics for CCSK processes]
Given a CCSK process $P$, $P \xrightarrow{\alpha} P'$ iff there is $k$ such that $P \arrof{\alpha[k]} P'$.
\end{definition}
We show now that the semantics above, defined on all CCSK processes, is indeed strictly related to the classical semantics of CCS defined in~\cite[Chapter 5]{CCS}\footnote{Actually, the semantics in~\cite[Chapter 5]{CCS} includes additional operators, as well as value passing. We consider its restriction to the operators we use in CCSK.}, which we denote as $\arroccs{}$.
Let $\delhist{}$ be the function that extracts the standard part of a process.
\begin{definition}[$\delhist{}$]
  The $\delhist{}$ function is inductively defined as follows:
  \[
  \begin{array}{rcl}
    \delhist{P} &=& P \textrm{ if } \std{P}\\
    \delhist{a[n].P} &=& \delhist{P}\\
    \delhist{P+Q} &=& \delhist{P} \textrm{ if }\neg \std{P}\\
    \delhist{P+Q} &=& \delhist{Q} \textrm{ if }\neg \std{Q}\\
    \delhist{P \parop Q} &=& \delhist{P} \parop \delhist{Q}\\
    \delhist{\res a P} &=& \res a \delhist{P}
  \end{array}
  \]
\end{definition}
The following proposition holds.

\begin{proposition}
Let $P,P'$ be CCSK processes, and $Q$ a CCS process.
If $P \xrightarrow{\alpha} P'$ then $\delhist{P} \arroccs{\alpha} \delhist{P'}$.
If $\delhist{P} \arroccs{\alpha} Q$ then there exists $P'$ in CCSK such that $P \xrightarrow{\alpha} P'$ and $Q=\delhist{P'}$.
\end{proposition}
\begin{proof}
  By rule inspection.\qed
\end{proof}

\subsubsection{Bisimilarities for CCS}
We start with the classical notion of CCS bisimilarity, extended as mentioned above.

\begin{definition}[Strong Bisimulation]\label{def:sbis}
  A symmetric relation $\rel$ on CCSK processes is a \emph{strong bisimulation} if
  whenever $P \rel Q$:
  \begin{itemize}
  \item\label{cond:sbis} if $P \xrightarrow{\alpha} P'$ then there exists $Q'$ such that $Q \xrightarrow{\alpha} Q'$ and $P' \rel Q'$.
  \end{itemize}
Let $\sim$ be the largest strong bisimulation.   
Two CCSK processes $P, Q$ are \emph{strongly bisimilar} if $(P,Q) \in \sim$. 
\end{definition}
We give below some examples of strongly bisimilar processes. 
\begin{example}[Strongly Bisimilar CCSK Processes]\label{ex:sbisp}\ 
    \begin{enumerate}
        \item $a+a \sim a$
        \item $a \parop a \sim a.a$
        \item\label{sccs:nomem} $a[m] \sim b[n]$  
        \item\label{sccs:exp} $a \parop b \sim a.b + b.a$
    \end{enumerate}
\end{example} 

Note that Item~\ref{sccs:nomem} shows that strong bisimilarity (like
all other CCS bisimilarities) abstracts away from the history.
Item~\ref{sccs:exp} is actually an instance of the Expansion
Law~\cite{milner1989communication}, a cornerstone of the theory of
classical CCS bisimilarity, whose general form is as follows:
\begin{eqnarray*}
P_1 \parop P_2 & =& \sum \{\alpha.(P'_1 \parop P_2) : P_1 \xrightarrow{\alpha} P'_1\} +
\sum \{\alpha.(P_1 \parop P'_2) : P_2 \xrightarrow{\alpha} P'_2\} +\\
&&\sum \{\tau.(P'_1 \parop P'_2) : P_1 \xrightarrow{\alpha} P'_1, P_2 \xrightarrow{\co{\alpha}} P'_2, \alpha \neq \tau\}
\end{eqnarray*}
where $\sum$ is $n$-ary choice.

It is well-known~\cite{phillips2007reversing} that the Expansion Law does not
hold for reversible calculi, and indeed we can provide a counterexample
using strong forward-reverse bisimilarity (cf.~Def.~\ref{def:frbis} and Ex.~\ref{ex:frbisp}).

In order to introduce weaker notions of bisimilarity we need the notation below.
Let \(\Rightarrow = (\xrightarrow{\tau}) ^*\) be the reflexive and transitive closure of \(\tau\) steps.

\begin{definition}[Weak Bisimulation]\label{def:wbis}
  A symmetric relation $\rel$ on CCSK processes is a \emph{weak bisimulation} if
  whenever $P \rel Q$:
  \begin{itemize}
  \item\label{cond:wbistau} if $P \xrightarrow{\tau} P'$ then there exists $Q'$ such that $Q \Rightarrow Q'$ and $P' \rel Q'$;
  \item\label{cond:wbis} if $P \xrightarrow{\alpha} P'$ with $\alpha \neq \tau$ then there exists $Q'$ such that $Q \Rightarrow\xrightarrow{\alpha}\Rightarrow Q'$ and $P' \rel Q'$.
  \end{itemize}
Let $\wsim$ be the largest weak bisimulation.   
Two CCSK processes $P, Q$ are \emph{weakly bisimilar} if $(P,Q) \in \wsim$. 
\end{definition}

\begin{example}[Weakly Bisimilar CCSK Processes]\label{ex:wbisp}\ 
    \begin{enumerate}
        \item $\tau.a \wsim a$
        \item $a+\tau.a \wsim a$
        \item\label{sccs:nomem} $a[m] \wsim b[n]$  
        \item\label{sccs:exp} $a \parop b \wsim a.b + b.a$
    \end{enumerate}
\end{example} 

We introduce also an intermediate notion, taken from~\cite{MontanariS91}, where $\tau$ steps need to be matched by at least one $\tau$ step.

\begin{definition}[Semi-Weak Bisimulation]\label{def:congbis}
  A symmetric relation $\rel$ on CCSK processes is a \emph{semi-weak bisimulation} if
  whenever $P \rel Q$:
  \begin{itemize}
      \item \label{cond:congbis} if $P \xrightarrow{\alpha} P'$ then there is $Q'$ such that $Q \Rightarrow \xrightarrow{\alpha} \Rightarrow  Q'$ and $P' \rel Q'$.
  \end{itemize}
  Let $\swsim$ be the largest semi-weak bisimulation.
Two CCSK processes $P, Q$ are semi-weak bisimilar if $(P,Q) \in \swsim$.
\end{definition}

\begin{example}[Semi-Weakly Bisimilar Processes]\label{ex:congbisp}\ 
    \begin{itemize}
        \item $a+a \swsim a$
        \item $a \parop a \swsim a.a $
        \item $a[n] \swsim b[m]$
        \item $a + \tau.a \swsim \tau.a$  
    \end{itemize}
\end{example} 

\subsubsection{Bisimilarities for CCSK}
The bisimulations below make sense only in reversible calculi, since they consider both forward and backward transitions.
We start with the notion of (revised) forward-reverse bisimulation from~\cite{lanese2021forward}.

\begin{definition}[Strong Forward-Reverse Bisimulation]\label{def:frbis}
  A symmetric relation $\rel$ is a \emph{strong forward-reverse bisimulation} (also called FR-bisimulation) if
  whenever $P \rel Q$:
  \begin{itemize}
  \item
  \label{cond:frbis} if $P \arro \vartheta x P'$ then there is $Q'$ such that $Q \arro \vartheta x Q'$ and $P' \rel Q'$.
  \end{itemize}
  Let $\frsim$ be the largest FR-bisimulation.
  Two CCSK processes $P$, $Q$ are FR-bisimilar if $(P,Q) \in \frsim$.
\end{definition}

\begin{example}[FR Bisimilar Processes]\label{ex:frbisp}\ 
    \begin{itemize}
        \item $a+a \frsim a$
        \item $a \parop a \not\frsim a.a $
        \item $a \parop b \not\frsim a.b + b.a $ 
    \end{itemize}
    As expected, instances of the Expansion Law do not hold any more.
\end{example} 

We now introduce notations to study the weak bisimulations in CCSK, and to manipulate \(\tau\) steps easily.
Let \(\Rightarrow _x = (\xrightarrow{\tau} _x) ^*\) the reflexive and transitive closure of \(\tau\) steps in the direction \(x\). 
Let mixed \(\tau \) reachability be \(\Rightarrow _m = (\xrightarrow{\tau}_f \cup \xrightarrow{\tau}_r) ^* \). This is an equivalence relation. Let \(P \Arro \mu {m,x} P' \) iff \( \exists Q,Q' \text{ s.t. } (P \Rightarrow _m Q \arro \mu x Q' \Rightarrow _m P') \).

We now define two variants of weak reversible bisimulation, which
differ in whether the $\tau$ steps used to match some action $\mu$ need to be
in the same direction as $\mu$ or not.
\begin{definition}[Weak Mixed Bisimulation]\label{def:mbis}
  A symmetric relation $\rel$ is a \emph{weak mixed bisimulation} (called mixed bisimulation) if
  whenever $P \rel Q$:
  \begin{enumerate}
    \item \label{cond:mtau} if $P \arro \tau x P'$ then there exists $Q'$ such that $Q \Rightarrow_m Q'$ and $P' \rel Q'$;
    \item \label{cond:mmu} if $P \arro \mu x P'$ then there exists $Q'$ such that $Q \Arro \mu {m,x} Q'$ and $P' \rel Q'$.
  \end{enumerate}
Let $\approx_m $ be the largest mixed weak bisimulation. 
Two CCSK processes $P, Q$ are \emph{weakly mixed bisimilar} if $(P, Q) \in \approx_m$. 
\end{definition}
Intuitively, weak mixed bisimilarity abstracts away all the \(\tau\) steps.

\begin{example}[Mixed Bisimilar Processes]\label{ex:mbisp}\ 
    \begin{itemize}
        \item $\tau.a \approx_m a$
        \item $\tau \parop a \approx_m a$
        \item $\tau + a \approx_m a$  
        \item $\tau[m] + a \approx_m a$
    \end{itemize}
\end{example} 
We also introduce a "directional" weak bisimulation, in which the \(\tau\) steps have to be in the same direction as the action \(\mu\).
Let \(P \Arro \mu {d,x} P'\) iff \( \exists Q,Q' \text{ s.t. } (P \Rightarrow _x Q \arro \mu x Q' \Rightarrow _x P') \).

\begin{definition}[Weak Directional Bisimulation]\label{def:dbis}
  A symmetric relation $\rel$ is a \emph{weak directional bisimulation} (called directional bisimulation) if
  whenever $P \rel Q$:
  \begin{enumerate}
    \item \label{cond:dtau} if $P \arro \tau x P'$ then there exists $Q'$ such that $Q \Rightarrow _x Q'$ and $P' \rel Q'$;
    \item \label{cond:dmu} if $P \arro \mu x P'$ then there exists $Q'$ such that $Q \Arro \mu {d,x} Q'$ and $P' \rel Q'$;
  \end{enumerate}
  Let $\approx _d$ be the largest directional bisimulation.
  Two CCSK processes $P, Q$ are directionally bisimilar if $(P, Q) \in \approx_d$.
\end{definition}

\begin{example}[Directionally Bisimilar Processes]\label{ex:dbisp}\ 
    \begin{itemize}
        \item $\tau.a \approx_d a$
        \item $\tau \parop a \approx_d a$
        \item $\tau + a \not\approx_d a$
    \end{itemize}
\end{example}

\section{Relations between Bisimilarities}\label{sec:relations}
We now compare the notions of bisimilarity introduced in the previous section.
Interestingly, the considered notions give rise to two hierarchies.

\begin{proposition}[Hierarchies of Bisimilarities]\label{prop:hierarchies}
\begin{enumerate}
    \item\label{lemma:ksbis_sub_sbis}\label{lemma:sbis_sub_swbis}\label{lemma:swbis_sub_wbis} $\frsim \subset \sim \subset \swsim \subset \wsim$;
    \item\label{lemma:kwdbis_sub_wbis} $\frsim \subset \approx_d \subset 
      \begin{cases}
        \approx_m \\ \wsim
      \end{cases}
        $.
\end{enumerate}
\end{proposition}
\begin{proof}
  Inclusions follow directly from the definitions.  The proof of their
  strictness will actually be deferred to
  Proposition~\ref{prop:nonemptystd}, providing witnesses for each of
  them.\qed
\end{proof}
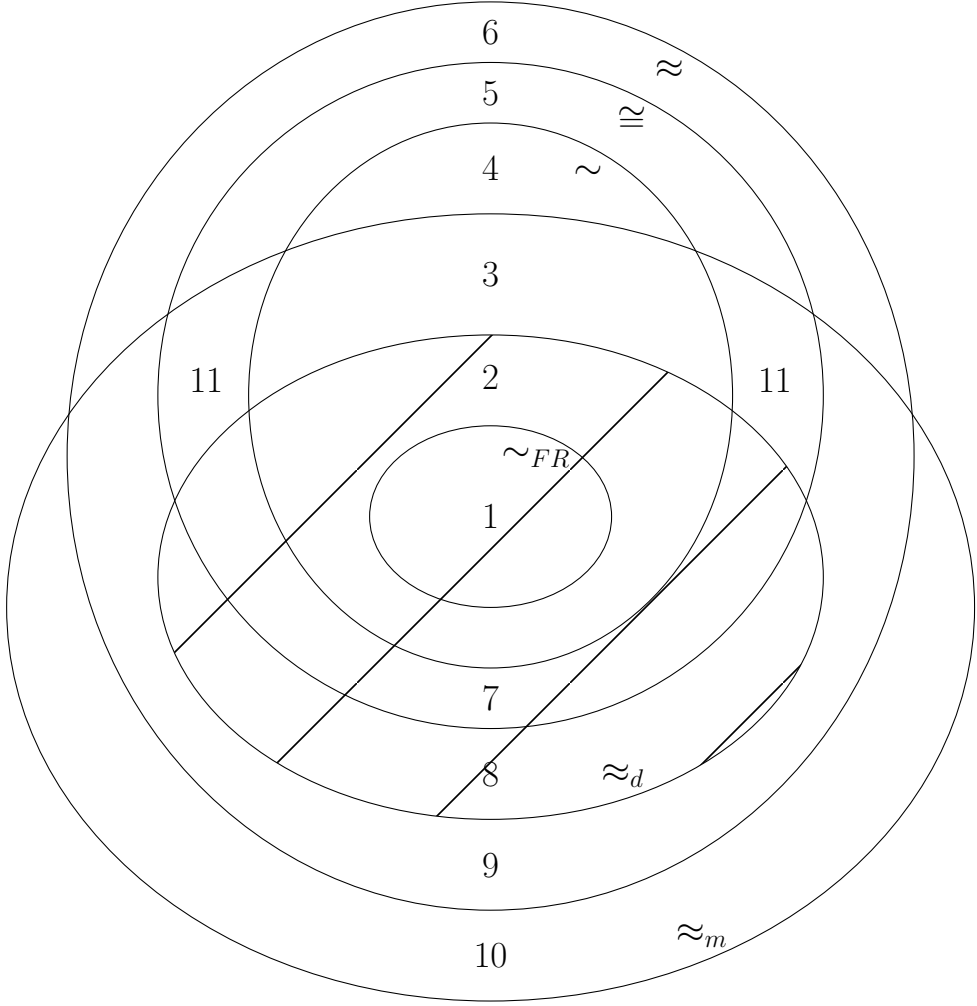
\begin{figure}[t]
\begin{tikzpicture}[scale=0.8, transform shape]

\tikzstyle{ensemble} = [draw, ellipse, minimum width=6cm, minimum height=6cm, align=center]

\node[ensemble, minimum width=4cm, minimum height=3cm, label={[font=\LARGE, label distance=-23pt]85:$\frsim$}] (ks) at (0,-2) {};

\node[ensemble, minimum width=8cm, minimum height=9cm, label={[font=\LARGE, label distance=-22pt]70:$\sim$}] (s) at (0,0) {};

\node[ensemble, minimum width=11cm, minimum height=11cm, label={[font=\LARGE, label distance=-24pt]65:$\swsim$}] (sw) at (0,0) {};

\node[ensemble, minimum width=14cm, minimum height=15cm, label={[font=\LARGE, label distance=-22pt]67:$\wsim$}] (w) at (0,-1) {};

\node[ensemble, minimum width=11cm, minimum height=8cm, postaction={pattern=north east spaced lines}, 
label={[font=\LARGE, label distance=-24pt]300:$\approx_d$}] (kwd) at (0,-3) {};

\node[ensemble, minimum width=16cm, minimum height=13cm, label={[font=\LARGE, label distance=-26pt]300:$\approx_m$},] (kwm) at (0,-3.5) {};
\node at (0,-2) {\LARGE1};
\node at (0,0.3) {\LARGE2};
\node at (0,2) {\LARGE3};
\node at (0,3.75) {\LARGE4};
\node at (0,5) {\LARGE5};
\node at (0,6) {\LARGE6};
\node at (0,-5) {\LARGE7};
\node at (0,-6.25) {\LARGE8};
\node at (0,-7.75) {\LARGE9};
\node at (0,-9.25) {\LARGE10};
\node at (4.7,0.25) {\LARGE11};
\node at (-4.7,0.25) {\LARGE11};
  
\end{tikzpicture}
\caption{Hierarchies of bisimilarities}\label{fig:hierarchy}
\end{figure}

The first hierarchy relates strong forward-reverse bisimilarity to CCS bisimilarities.
The second hierarchy instead focuses on reversible bisimilarities. Note that $\approx_d$ is included in both $\approx_m$ and $\wsim$ (hence in their intersection).
The two hierarchies are graphically represented in Fig.~\ref{fig:hierarchy}.
The relations are written just inside the elliptical set they represent. 
The lines on the set for $\approx_d$ are just a visual help to distinguish the corresponding ellipse.

We now show that there are no other inclusions beyond the ones in
Proposition~\ref{prop:hierarchies}, and that all the inclusions there
are actually strict. Graphically, it means that all the areas in
Figure~\ref{fig:hierarchy} are not empty. We show this by providing
examples of pairs of processes in each of them.  Notably, all the
examples are made of standard processes, hence none of these notions
collapse when restricting the attention to standard processes. In
other words, all these notions induce different equivalence relations
on CCS processes.

\begin{proposition}[Hierarchies are Strict on Standard Processes]\label{prop:nonemptystd}
  All the areas in Fig.~\ref{fig:hierarchy} are not empty, and each of
  them contains at least a pair of standard processes.
\end{proposition}
\begin{proof}
  The ones below are witnesses for every area. 
  \begin{enumerate}
    \item $\frsim\ :\ (a+a,a)$. The two $a$s are indistinguishable.
    \item $(\sim\cap\approx_d)\backslash\frsim\ :\ (\tau.\tau.(a.b+b.a)+\tau.(a\parop b),\ \tau.\tau.(a\parop b)+\tau.(a.b+b.a))$. We have that $a.b+b.a \sim a\parop b$ is an instance of the expansion law, but the same equivalence is not valid for $\frsim$. Under $\approx_d$ one can use equal subterms to match the challenge since they can be reached by taking a different number of $\tau$ steps.
    \item\label{item:notd} $(\sim\cap\approx_m)\backslash\approx_d\ :\ (\tau.((a.b+b.a)+(c\parop d))+\tau.((a\parop b)+\tau.(c.d+d.c)),\ \tau.((a\parop b)+\tau.(c\parop d))+\tau.((a.b+b.a)+(c.d+d.c))$. The equivalence holds under $\sim$ thanks to the expansion law. It holds also under $\approx_m$ since the choice of which $\tau$ to execute can always be undone to select the desired branch. This is not the case under $\approx_d$ where on the left one can reach a state where the only forward actions enabled are from $c.d+d.c$, while on the right no such state exists (if $c.d+d.c$ is forward enabled, then also $a.b+b.a$ is enabled).  
    \item $\sim\backslash\approx_m\ :\ (a\parop b,\ a.b+b.a)$. We use again an instance of the expansion law.
    \item $\swsim\backslash(\sim\cup\approx_m)\ :\ (b\parop (a+a.\tau),\ a.(\tau\parop b)+a.b+b.a.\tau)$. This fails under $\approx_m$, since the right hand side has no $a \parop b$ to match the left hand side one. This fails under $\sim$ since on the right if one starts from $b$, there is no way to avoid a $\tau$, while this can be avoided on the left. This is not an issue under $\swsim$ since the $a$ without $\tau$ can be matched by executing both $a$ and $\tau$. 
    \item $\wsim\backslash(\swsim\cup\approx_m)\ :\ (a.\tau\parop b,\ a.b+b.a)$. This holds under $\wsim$ thanks to the expansion law and since the $\tau$ can be abstracted away. Instead, the expansion law fails under $\approx_m$ and the $\tau$ needs to be matched by another $\tau$ under $\swsim$. 
    \item $(\swsim\cap\approx_d)\backslash\sim\ :\ (a+a.\tau,\ a.\tau)$. This fails under $\sim$ since on the left executing an $a$ leads to a state where no $\tau$ can be performed. This is not an issue for $\approx_d$ where the $\tau$ can be matched by staying idle. For $\swsim$, left $a$ can be matched by executing both $a$ and $\tau$. 
    \item $\approx_d\backslash\swsim\ :\ (\tau,\ 0)$. This fails under $\swsim$ since the $\tau$ cannot be matched, instead under $\approx_d$ the $\tau$ can be matched by staying idle. 
    \item $(\wsim\cap\approx_m)\backslash(\swsim\cup\approx_d)\ :\newline (\tau.((a.b+b.a)+(c\parop d))+\tau.((a\parop b)+\tau.(c.d+d.c))+\tau,\ \tau.((a\parop b)+\tau.(c\parop d))+\tau.((a.b+b.a)+(c.d+d.c))+\tau.\tau$. This holds under $\wsim$ thanks to the expansion law, and since $\tau$ and $\tau.\tau$ are weakly bisimilar. However, the latter are not semi-weakly bisimilar, hence $\swsim$ fails. Also, this holds under $\approx_m$, since it abstracts away from $\tau$ actions. $\approx_d$ fails as well, for the same reason as in item~\ref{item:notd}. 
    \item $\approx_m\backslash\wsim\ :\ (a,\ a+\tau)$. This is well-known not to hold under $\wsim$. Instead under $\approx_m$ the $\tau$ step can be mimicked by staying idle since the $a$ action remains enabled also after the $\tau$, since the $\tau$ can be undone to do $a$.
    \item $(\swsim\cap\approx_m)\backslash(\sim\cup\approx_d)\ :\ (\tau.(a.b+b.a+\tau+\tau.\tau)+(a\parop b),\tau.(a.b+b.a+(a\parop b)+\tau.\tau))$. This fails under $\sim$ since $\tau.\tau$ is not matched on the right, since afterwards a third $\tau$ would be enabled. This fails under $\approx_d$ since after the first $\tau$ we can still go to $(a\parop b)$ on the right, but not on the left.
      This succeeds under $\approx_m$ since $\tau$s are abstracted away, and terms without $\tau$s are identical. This succeeds under $\swsim$ since $\tau$s can always be matched, and thanks to the expansion law.   \qed
\end{enumerate}
\end{proof}
If instead of considering only pairs of standard processes we consider only
pairs of non-standard processes, then all the areas remain non-empty.
\begin{proposition}
  Each of the areas in Fig.~\ref{fig:hierarchy} contains at least a pair of non-standard processes.
\end{proposition}
\begin{proof}
   One can take the witnesses from the proof of Proposition~\ref{prop:nonemptystd} add a $\tau[n]$ prefix in front of both processes to obtain witnesses made of non-standard processes. \qed
\end{proof}
Finally, if we consider pairs made of a non-standard process and a standard one, then $\frsim$ becomes empty. The other areas remain non-empty.
\begin{proposition}
  Each of the areas in Fig.~\ref{fig:hierarchy} contains at least a pair made of a non-standard process and a standard one, but for the area 1, corresponding to $\frsim$.
\end{proposition}
\begin{proof}
   The area for $\frsim$ is empty since $\frsim$ can always distinguish a non-standard process, that can make a backward move, from a standard one, that cannot.
   For the others, one can take the witnesses from the proof of Proposition~\ref{prop:nonemptystd} add a $\tau[n]$ prefix in front of only one of the two components. This preserves all the CCS bisimilarities, which cannot observe the history, as well as the weak CCSK bisimilarities, where the backward $\tau[n]$ can be matched by the other process by staying idle.  \qed
\end{proof}


\section{Congruence Properties of Bisimilarities}\label{sec:properties}
We first discuss whether the considered equivalences are congruences
or not, namely in the case of strong CCS bisimilarity whether $P \sim
Q \implies C[P]\sim C[Q]$. Note that this implication makes sense only
if all the involved processes are well-formed, hence we only consider
this case.

Strong bisimilarity is a congruence in
CCS~\cite{milner1989communication}. Somehow surprisingly, its
extension to CCSK is not a congruence, as shown in the counterexample
below.

\begin{example}[$\sim$ is not a congruence]\label{ex:counter}
\[\tau[n]\sim \nil \centernot\implies a+\tau[n]\sim a+\nil\]
Processes on the left are strongly bisimilar since $\sim$ abstracts away from the history.
Processes on the right are not strongly bisimilar since $a+\tau[n]$ cannot execute any forward move, while $a+\nil$ can execute $a$.
\end{example}
The key point here is that forward equivalences abstract away from
the history, but adding a choice where the added branch is a non-standard process
disables the other branch (which needs to be standard to ensure
well-formedness). Thus, the same issue also occurs for the other
forward equivalences we consider, namely weak ($\wsim$) and semi-weak
($\swsim$) bisimilarities, as stated below.

\begin{proposition}[Forward equivalences are not congruences in CCSK]\label{prop:fnocongr}
None of $\sim$, $\wsim$ and $\swsim$ are congruences on CCSK terms.
\end{proposition}
\begin{proof}
Counterexample~\ref{ex:counter} proves the thesis for all the equivalences.\qed
\end{proof}
Note that for $\wsim$ the usual problem of CCS that $\tau.a\approx a \centernot\implies \tau.a+b\approx a+b$ remains. However, $\swsim$ is a congruence in CCS \cite{MontanariS91}, but not in CCSK for the reason above.

Concerning reversible equivalences, $\frsim$ has been proved to be a
congruence in~\cite[Proposition 4.9]{lanese2021forward}. Instead, for
weak directional bisimilarity a problem similar to the one above occurs.

\begin{proposition}[Weak directional bisimilarity is not a congruence]
$\approx_d$ is not a congruence on CCSK terms.
\end{proposition}
\begin{proof}
The following counterexample proves the thesis:\\
$\tau[n]\approx_d\nil \centernot\implies a+\tau[n]\approx_da+0$\qed
\end{proof}

This is not the case for weak mixed bisimilarity, which, somehow
surprisingly is a congruence. In order to clarify why this is the case, we first show a property of $\approx_m$
which rules out counterexamples as the ones above.

\begin{proposition}\label{prop:almoststd}
Assume $P \approx_m Q$ where $P$ is standard. Then $Q \Rightarrow_m
Q'$ with $Q'$ standard.
\end{proposition}
\begin{proof}
Assume towards a contradiction that there is no such $Q'$. By
definition, we have $Q (\arrob{\theta})^* \tostd{Q}$. At least one of the steps is not
a $\tau$, otherwise we would have proven the thesis. Let us take the first such action. Then $Q$ can perform a backward non-$\tau$ action. However, such an action cannot be matched by $P$, since $P$ is standard (and executing $\tau$ steps only enables backward $\tau$ steps, while we need to match a non-$\tau$ backward action). \qed
\end{proof}

We also show that indeed weak mixed bisimilarity completely abstracts away from $\tau$ steps.
\begin{proposition}[Weak mixed bisimilarity abstracts away from $\tau$ steps]\label{prop:mixedtau}\mbox{}\\
 $P \Rightarrow_m Q$ implies $P \approx_m Q$.
\end{proposition}
\begin{proof}
Thanks to the Loop Lemma (cf.~\cite[Prop. 5.1]{phillips2007reversing}), we also have $Q \Rightarrow_m P$.
Hence, any challenge from $P$ can be matched by $Q$ by first reducing to $P$, and vice versa.\qed 
\end{proof}
Note that even if $P \Rightarrow_m Q$ and $P \approx_m Q$ are both
equivalence relations, they do not coincide. E.g., $a+a \approx_m a$, but $a+a {\centernot\Rightarrow}_m a$.

\begin{theorem}[Weak mixed bisimilarity is a congruence]\label{th:mixediscongr}
$\approx_m$ is a congruence on CCSK terms.
\end{theorem}
\begin{proof}
We prove the thesis by induction on the structure of the context, with
a case for each operator. Cases for prefix and restriction are trivial. Let us consider choice and parallel composition.
\begin{description}
\item[Choice:] we have to show that if $P \approx_m Q$ then $P + R \approx_m Q + R$. Note that at most one among $P$ and $R$ can be non-standard due to well-formedness. Assume first both $P$ and $R$ are standard. If the challenge is from $P$, then $Q$ inside $Q+R$ can match the challenge with the same sequence of moves used in $Q$ alone, all lifted thanks to rule (CHOICE). If $R$ moves, and $Q$ is standard, then the very same moves can be performed in both the cases, again using rule (CHOICE) to lift them. If $Q$ is not standard, thanks to Proposition~\ref{prop:almoststd} above, we can first reduce $Q$ to $\tostd{Q}$, and then match the moves from $R$ as above. Note that $Q$ and $\tostd{Q}$ are mixed bisimilar thanks to Proposition~\ref{prop:mixedtau}, hence the reduction preserves mixed bisimilarity.
Assume now $P$ is non-standard. Then only $P$ can move, and transitions can be lifted using rule (CHOICE) since by well-formedness $R$ is standard. Assume now $R$ is non-standard. Analogously to the above, only $R$ can move, in both the cases since $P$ and $Q$ need to be standard due to well-formedness.
\item[Parallel composition:] we have to show that if $P \approx_m Q$ then $P \parop R \approx_m Q \parop R$. Assume $P \parop R \arrof{\alpha}$. There are three subcases depending on which component contributes to the transition.
\begin{description}
\item[Transition from $P$:] $Q$ can match the transition, and the matching computation can be lifted to $Q \parop R$ thanks to rule (PAR).
\item[Transition from $R$:] the same transitions can be done on both the sides, remaining in the relation.
\item[Synchronization:] $Q$ can match transitions from $P$ by hypothesis, and transitions from $R$ can be performed on both the sides. This includes the components of the transitions that give rise to the synchronization. Hence we stay in the relation.
\end{description}
The case of backward transitions is analogous.
\end{description}
\qed
\end{proof}
We believe that the notion of mixed bisimilarity is very
relevant. Indeed it provides a notion of bisimilarity which completely
abstracts away from $\tau$ actions, which is coinductive (since it can
be formulated as a bisimulation), and which is a congruence. We are
not aware of any other notion of bisimilarity which has all these
properties.

While leaving a more detailed analysis of this equivalence for future
work, we discuss here some relevant axioms enabling to axiomatically
reason on this equivalence. Notice that it makes sense to discuss
about axioms since mixed bisimilarity is a congruence.

Various correct axioms for $\frsim$ have been proposed
in~\cite[Theorem ~4.10]{lanese2021forward}.  While trivially all these
axioms are correct also for $\approx_m$, we focus here
on axioms which are specific of weak mixed bisimilarity.

The axioms in Fig.~\ref{fig:axiomatisation_approx_m} hold for weak mixed bisimilarity.
\begin{theorem}\label{th:axiom_approx_m}
  The axioms in Figure~\ref{fig:axiomatisation_approx_m} are correct w.r.t. weak mixed bisimilarity.
\end{theorem}
\begin{proof}
 $\tau$ moves can always be matched by the other process by staying idle, while moves from $P$ can be matched by first doing or undoing $\tau$ steps as needed. Note that executing $\tau$-steps moves from the top-3 rows to the bottom ones.
 \qed
\end{proof}

\begin{figure}[t]
\begin{align}
  \tag{TAU-PREF-M}  \tau.P &= P\\
  \tag{TAU-CH-M}  \tau + P &= P\\
  \tag{TAU-PAR-M}  \tau \parop P &= P\\
  \tag{TAU-PREF-K} \tau[n].P &= P\\
  \tag{TAU-CH-K}  \tau[n] + P &= P\\
  \tag{TAU-PAR-K}  \tau[n] \parop P &= P
\end{align}
\caption{CCSK axioms for weak mixed bisimilarity $\approx_m$}\label{fig:axiomatisation_approx_m}
\end{figure}

The axioms in Fig.~\ref{fig:axiomatisation_approx_m} are aligned with
Proposition~\ref{prop:mixedtau} in showing that mixed bisimilarity
completely abstracts away from $\tau$ steps.

We remark that as shown in Figure~\ref{fig:hierarchy}, axioms which
hold for weak bisimilarity $\wsim$ do not necessarily hold for
$\approx_m$. Let us now discuss the well-known Milner $\tau$-laws of
weak bisimilarity (collected in Fig.~\ref{fig:axiomatisation_approx}).
The laws (TAU-CH) and (TAU-SEQ) follow directly from (TAU-PREF-M), and
idempotence of + for the former.

Instead (TAU-DUPL-CH) fails, as shown below.

\begin{example}[(TAU-DUPL-CH) does not hold]
Consider the right-hand side challenge:
\[\alpha.(P +\tau.Q) +\alpha.Q \arrof{\alpha[n]} \alpha.(P +\tau.Q) +\alpha[n].Q\]
There are two possible answers from the left-hand side, namely:
\begin{eqnarray*}
\alpha.(P +\tau.Q) &&\arrof{\alpha[n]} \alpha[n].(P +\tau.Q)\\
                   &&\arrof{\alpha[n]}\arrof{\tau[m]} \alpha[n].(P +\tau[m].Q)
\end{eqnarray*}
(We can also reach the same states after having undone and redone multiple times the $\tau$ step.)
In both the cases, actions from $P$ are enabled, directly in the first
case, and by first undoing $\tau[m]$ in the second case. However,
no action from $P$ can be executed in the right-hand side above, since
we need to undo one $\alpha[n]$ and do $\alpha$ on the other side.
\end{example}

\begin{figure}[t]
\begin{align}
  \tag{TAU-CH}  P+\tau.P &= \tau.P\\
  \tag{TAU-SEQ} \alpha.\tau.P &= \alpha.P\\
  \tag{TAU-DUPL-CH} \alpha.(P +\tau.Q) &= \alpha.(P +\tau.Q) +\alpha.Q
\end{align}
\caption{$\tau$ laws for weak bisimilarity $\approx$}\label{fig:axiomatisation_approx}
\end{figure}

\section{Conclusion and Future Work}
In this paper, we contrasted different notions of bisimilarity for
CCSK processes, including two definitions of weak bisimilarities not
previously discussed in the literature. We also proved that none of
these notions are equivalent, not even if we restrict to standard
processes only. Notably, weak mixed bisimilarity turns out to be
coinductive, to be a congruence, and to completely abstract away from
$\tau$-steps, making it a very interesting equivalence.

%
We hope that these results can be the basis of a more detailed study
of bisimilarities in CCSK. We remark that such a deeper understanding
may also impact classical concurrency theory, since there are strong
relations~\cite{AubertC20,AubertPU26} between reversible strong
bisimilarities and history-preserving~\cite{GlabbeekG01} and
hereditary history-preserving~\cite{hhp} bisimilarities.  Also, weak
mixed bisimilarity induces an equivalence on CCS, which is a
congruence and abstracts away from $\tau$ steps. Notice however that
its current definition is not coinductive in CCS, since it relies on
CCSK terms.


We present now a few other items for future research.  First, we
remark that while giving correct axioms for weak mixed bisimilarity,
we have not provided a complete axiomatization. This is definitely a
relevant item for future work (given the interesting properties of
weak mixed bisimilarity) but not easy. Indeed, there are no complete
axiomatizations for forward-reverse bisimilarity either, which we
expect to be needed as first step before tackling the mixed case.

Another interesting item would be to understand how other
bisimilarities, and more in general behavioral equivalences, can be
extended from CCS to CCSK. Given that reversibility provides quite a
strong observational power (as shown by the relations with
history-preserving and hereditary history-preserving bisimilarities),
it may be the case that some of them collapse.

\bibliographystyle{plain} 
\bibliography{biblio} 


\end{document}